\documentclass[11pt]{myclass-arXiv}

\usepackage{graphicx} 
\usepackage{subfigure} 

\usepackage[charter]{mathdesign}
\usepackage{algorithm}
\usepackage{algorithmic}

\newcommand{\eps}{\varepsilon}
\renewcommand{\epsilon}{\varepsilon}
\renewcommand{\b}[1]{\ensuremath{\mathbb{#1}}}

\newcommand{\ensuretext}{\textrm}

\newcommand{\PC}{\ensuremath{\textsc{PlaceCenter}}\xspace}
\newcommand{\place}{\ensuremath{\textsc{Place}}\xspace}
\newcommand{\recenter}{\ensuremath{\textsc{Recenter}}\xspace}
\newcommand{\JL}{\textsf{JL}\xspace}
\newcommand{\REE}{\ensuremath{\textsf{REE}}\xspace}

\newcommand{\fmds}{\ensuretext{f\textsf{MDS}}\xspace}
\newcommand{\cmds}{\ensuretext{c\textsf{MDS}}\xspace}
\newcommand{\rmds}{\ensuretext{r\textsf{MDS}}\xspace}
\newcommand{\rsmds}{\ensuremath{r^2}\ensuretext{\textsf{MDS}}\xspace}
\newcommand{\allsmds}{\{c,g\}-\{1,2\}-s\textsf{MDS}\xspace}
\newcommand{\smds}{\ensuretext{\textsf{SMDS}}\xspace}
\newcommand{\mds}{\ensuretext{\textsf{MDS}}\xspace}
\newcommand{\err}{\textsf{Err}\xspace}

 \usepackage{hyperref}

\title{A Unified Algorithmic Framework for Multi-Dimensional Scaling}
\author{
Arvind Agarwal\thanks{Partially supported by NSF IIS-0712764}
\and
Jeff M. Phillips\thanks{Supported by a subaward to the University of Utah under NSF award 0937060 to the Computing Research Association}
\and
Suresh Venkatasubramanian\thanks{Partially supported by NSF CCF-0953066}
}

\date{}
\begin{document}

\maketitle
\begin{abstract}
In this paper, we propose a unified algorithmic framework for solving many known variants of \mds. Our algorithm is a simple iterative scheme with guaranteed convergence, and is \emph{modular}; by changing the internals of a single subroutine in the algorithm, we can switch cost functions and target spaces easily. In addition to the formal guarantees of convergence, our algorithms are accurate; in most cases, they converge to better quality solutions than existing methods, in comparable time. We expect that this framework will be useful for a number of \mds variants that have not yet been studied. 

Our framework extends to embedding high-dimensional points lying on a sphere to points on a lower dimensional sphere, preserving geodesic distances.  As a compliment to this result, we also extend the Johnson-Lindenstrauss Lemma to this spherical setting, where projecting to a random $O((1/\eps^2) \log n)$-dimensional sphere causes $\eps$-distortion.  
\end{abstract}


\section{Introduction}

Multidimensional scaling (\mds)~\cite{KW78,CoxBook,BG05} is a widely used method for embedding a general distance matrix into a low dimensional Euclidean space, used both as a preprocessing step for many problems, as well as a visualization tool in its own right.  \mds has been studied and used in psychology since the 1930s~\cite{YH38,Tor52,kruskal} to help visualize and analyze data sets where the only input is a distance matrix.  More recently \mds has become a standard dimensionality reduction and embedding technique to manage the complexity of dealing with large high dimensional data sets \cite{Cayton,CB09,Pless01embeddingimages,BBK08}.

In general, the problem of embedding an arbitrary distance matrix into a fixed dimensional Euclidean space with minimum error is nonconvex (because of the dimensionality constraint). Thus, in addition to the standard formulation~\cite{deL77}, many variants of \mds have been proposed, based on changing the underlying error function~\cite{YH38,Cayton}. There are also applications where the target space, rather than being a Euclidean space, is a manifold (e.g. a low dimensional sphere), and various heuristics for \mds in this setting have also been proposed~\cite{dLM09,BBK08}. 

Each such variant is typically addressed by a different heuristic, including majorization, the singular value decomposition, semidefinite programming, subgradient methods, and standard Lagrange-multipler-based methods (in both primal and dual settings). Some of these heuristics are efficient, and others are not; in general, every new variant of \mds seems to require different ideas for efficient heuristics.

\subsection{Our Work}
\label{sec:our-work}

In this paper, we present a unified algorithmic framework for solving many
variants of \mds. Our approach is based on an iterative local improvement
method, and can be summarized as follows: ``Pick a point and move it so
that the cost function is locally optimal.  Repeat this process until
convergence.''  The improvement step reduces to a well-studied and efficient family of iterative minimization techniques, where the specific algorithm depends on the variant of \mds.  

A central result of this paper is a single general convergence result for all variants of \mds that we examine. This single result is a direct consequence of the way in which we break down the general problem into an iterative algorithm combined with a point-wise optimization scheme. Our approach is generic, efficient, and simple. The high level framework can be written in 10-12 lines of MATLAB code, with individual function-specific subroutines needing only a few more lines each. Further, our approach compares well with the best methods for all the variants of \mds. In each case our method is consistently either the best performer or is close to the best, regardless of the data profile or cost function used, while other approaches have much more variable performance. Another useful feature of our method is that it is parameter-free, requiring no tuning parameters or Lagrange multipliers in order to perform at its best. 

\paragraph{Spherical \mds.}

An important application of our approach is the problem of performing \emph{spherical} \mds. Spherical \mds is the problem of embedding a matrix of distances onto a (low-dimensional) sphere. Spherical \mds has applications in texture mapping and image analysis~\cite{BBK08}, and is a generalization of the spherical \emph{dimensionality reduction} problem, where the goal is to map points from a high dimensional sphere onto a low-dimensional sphere. This latter problem is closely related to dimensionality reduction for finite dimensional distributions. A well-known  isometric embedding takes a distribution represented as a point on the $d$-dimensional simplex to the $d$-dimensional sphere while preserving the Hellinger distance between distributions. A spherical dimensionality reduction result is an important step to representing high dimensional distributions in a lower-dimensional space of distributions, and will have considerable impact in domains that represent data natively as histograms or distributions, such as in document processing~\cite{Pereira93distributionalclustering,Joachims/02a,944937}, image analysis~\cite{996342,1069007} and speech recognition~\cite{1163421}.

Our above framework applies directly to this setting, where for the local improvement step we adapt a technique first developed by Karcher for finding geodesic means on a manifold. In addition, we prove a Johnson-Lindenstrauss-type result for the sphere; namely, that $n$ points lying on a $d$-dimensional sphere can be embedded on a $O((1/\eps^2)\log n)$-dimensional sphere while approximately preserving the geodesic distances between pairs of points, that is, no distance changes by more than a relative $(1+\eps)$-factor. This latter result can be seen as complementary to the local improvement scheme; the formal embedding result guarantees the error while being forced to use $\log n$ dimensions, while the local improvement strategy generates a mapping into any $k$ dimensional hypersphere but provides no formal guarantees on the error. 

\paragraph{Summary of contributions.}

The main contributions of this paper can be summarized as follows:
\begin{itemize}
\item In Section~\ref{sec:algo} we present our iterative framework, illustrate how it is applied to specific \mds variants and prove a convergence result.
\item In Section~\ref{sec:expts} we present a comprehensive experimental study that compares our approach to the prior best known methods for different \mds variants. 
\item In Section~\ref{sec:sphere} we prove a formal dimensionality reduction result that embeds a set of $n$ points on a high-dimensional sphere into a sphere of dimension $O((1/\eps^2) \log n)$ while preserving all distances to within relative error of $(1+\epsilon)$ for any $\epsilon > 0$.
\end{itemize}

\section{Background and Existing Methods}
Multidimensional scaling is a \emph{family} of methods for embedding a distance matrix into a low-dimensional Euclidean space. There is a general taxonomy of \mds methods~\cite{CoxBook}; in this paper we will focus primarily the metric and generalized \mds problems.  

The traditional formulation of \mds~\cite{KW78} assumes that the distance matrix $D$ arises from points in some $d$-dimensional Euclidean space. Under this assumption, a simple transformation takes $D$ to a matrix of \emph{similarities} $S$, where $s_{ij} = \langle x_i, x_j \rangle$. These similarities also arise from many psychology data sets directly~\cite{Tor52,YH38}.  The problem then reduces to finding a set of points $X$ in $k$-dimensional space such that $XX^T$ approximates $S$. This can be done optimally using the top $k$ singular values and vectors from the singular value decomposition of $S$.  

A more general approach called SMACOF that drops the Euclidean assumption uses a technique known as stress majorization~\cite{MO79,deL77,dLM09}. It has been adapted to many other \mds variants as well including restrictions of data to lie on quadratic surfaces and specifically spheres~\cite{dLM09}. 

Since the sum-of-squares error metric is sensitive to outliers, Cayton and Dasgupta~\cite{Cayton} proposed a robust variant based on an $\ell_1$ error metric. They separate the rank and cost constraints, solving the latter using either semidefinite programming or a subgradient heuristic, followed by a singular value decomposition to enforce the rank constraints.

Many techniques have been proposed for performing spherical \mds. Among them are majorization methods (\cite{rank} and SMACOF-Q~\cite{dLM09}), a multiresolution approach due to Elad, Keller, and Kimmel~\cite{DBLP:conf/scalespace/ElbazKK05} and an approach based on computing the classical \mds and renormalizing~\cite{Pless01embeddingimages}.

\paragraph{Embeddings that guarantee bounded error.}
\label{sec:embedd-that-guar}
A complementary line of work in dimensionality reduction fixes an error bound for \emph{every} pair of distances (rather than computing an average error), and asks for the minimum dimension a data set can be embedded in while maintaining this error. The Johnson-Lindenstrauss Lemma~\cite{JL84} states that any collection of $n$ points in a Euclidean space can be embedded in a $O((1/\eps^2) \log n)$ dimensional Euclidean space that preserves all distances within a relative error of $\eps$. If the points instead define an abstract metric space, then the best possible result is an embedding into $O(\log n)$-dimensional Euclidean space that preserves distances up to a factor of $O(\log n)$. An exhaustive survey of the different methods for dimensionality reduction is beyond the scope of this paper - the reader is directed to the survey by Indyk and Matousek for more information~\cite{indykbook}.

The Johnson-Lindenstrauss Lemma can be extended to data lying on manifolds. Any manifold $M$ with ``linearization dimension'' $k$ (a measure of its complexity) can be embedded into a $O((1/\eps^2)k \log(kn))$ dimensional space so that all pairwise \emph{Euclidean} distances between points on $M$ are distorted by at most a relative $(1+\eps)$-factor~\cite{AHY07,Sar06,Mag02}.  A $k$-dimensional sphere has linearization dimension $O(k)$, so this bound applies directly for preserving the chordal (i.e. Euclidean) distance between points on a sphere.  The geodesic distance between points on a sphere can be interpreted as the angle between the points in radians, and a result by Magen~\cite{Mag02} show that $O((1/\eps^2) \log n)$ dimensions preserve angles to within a relative factor of $1 + \sqrt{\eps}$ (which is weaker than our result preserving the geodesic distance to within a relative factor of $(1+\eps)$).

\section{Definitions}
\label{sec:definitions}

Let $D = (d_{ij})$ be an $n\times n$ matrix representing distances between all pairs of points in a set $Y = \{y_1, \ldots y_n\}$. In general, we assume that $D$ is symmetric $d_{ij} = d_{ji}$, although our method does not formally require this. The multidimensional scaling problem takes as input $Y$, $D$ and $k$, and asks for a mapping $\mu : Y \rightarrow X$ from $Y$ to a set of points $X$ in a $k$-dimensional space $T$ such that the difference between the original and resulting distances is minimized. 

There are many different ways to measure the difference between the sets of distances, and these can be captured by the following general function
\[ 
C(X, D) = \sum_i \sum_j \err(f(x_i,x_j) - d_{ij}) 
\]
where \err measures the discrepancy between the source and target distances, and $f$ denotes the function that measures distance in the target space. 
\begin{itemize}
 \item $T = \reals^k, \err(\delta) = \delta^2, f(x, x') = \|x - x'\|_2$: This is a general form of the \mds problem, which we refer to as \fmds.
\item $T = \reals^k, \err(\delta) = |\delta|, f(x, x') = \|x - x'\|_2$: This is a \emph{robust} variant of \mds called \rmds, first suggested by Cayton and Dasgupta~\cite{Cayton}.
\item $T = \b{S}^k, \err(\delta) = |\delta|$ or $\delta^2$, $f(x, x')$ is either chordal (c) or geodesic distance (g) on $\b{S}^k$. We refer to this family of problems as \allsmds.
\end{itemize}

It will be convenient to split the expression into component terms. We define 
\[ 
C_i(X, D,x_i) = \sum_j \err(f(x_i,x_j) - d_{ij}) 
\]
which allows us to write $C(X,D) = \sum_i C_i(X, D,x_i)$. 

\paragraph{Notes.}
The actual measure studied by Cayton and Dasgupta~\cite{Cayton} is not \rmds. It is a variant which takes the absolute difference of the \emph{squared} distance matrices. We call this measure \rsmds. Also, classical \mds does not appear in this list since it tries to minimize the error between similarities rather than distances. We refer to this measure as \cmds.

\section{Algorithm}
\label{sec:algo}

We now present our algorithm $\PC(X,D)$ that finds a mapping $Y \rightarrow X$ minimizing $C(X, D)$.  For now we assume that we are given an initial embedding $X_1 \in \b{R}^k$ to seed our algorithm.   Our experiments indicate the SVD-based approach~\cite{YH38} is almost always the optimal way to seed the algorithm, and we use it unless specifically indicated otherwise.  

\begin{algorithm}
\caption{\PC(D)}
  \begin{algorithmic}
    \STATE Run any \mds strategy to obtain initial seed $X$.
    \REPEAT
      \STATE $\epsilon \leftarrow C(X,D)$
      \FOR{$i = 1$ \textbf{to} $n$}
        \STATE $x_i \leftarrow \place_i(X,D)$ \{this updates $x_i \in X$\}
      \ENDFOR 
    \UNTIL{($\epsilon - C(X,D) < t$)}  \{for a fixed threshold $t$\}
  \STATE \textbf{return} $X$
  \end{algorithmic}
\end{algorithm}

\PC operates by employing a technique from the block-relaxation class of heuristics. The cost function can be expressed as a sum of costs for each point $x_i$, and so in each step of the inner loop we find the best placement for $x_i$ while keeping all other points fixed, using the algorithm $\place_i(X, D)$. A key insight driving our approach is that $\place_i(X, D)$ can be implemented either iteratively or exactly for a wide class of distance functions.  The process terminates when over all $i$, invoking $\place_i(X,D)$ does not reduce the cost $C(X,D)$ by more than a threshold $t$. The algorithm takes $O(n^2)$ for each iteration, since $\place_i(X,D)$ will take $O(n)$ time and computing $C(X,D)$ takes $O(n^2)$ time.

\subsection{A Geometric Perspective On $\place_i(X, D)$}
The routine $\place_i(X,D)$ is the heart of our algorithm. This routine  finds the optimal placement of a fixed point $x_i$ with respect to the cost function $C_i(X,D, x_i) =$    $\sum_j \err(f(x_i, x_j) - d_{ij})$. Set $r_j = d_{ij}$. Then the optimal placement of $x_i$ is given by the point $x^*$ minimizing the function 
\[
g(x) = \sum_j \err(f(x, x_j) - r_j).
\] 
Note that the terms $f(x, x_i)$ and $r_i = d_{ii}$ are zero, so we can ignore their presence in the summation for ease of notation. 

There is a natural geometric interpretation of $g(x)$, illustrated in Figure~\ref{fig:place}. Consider a sphere around the point $x_j$ of radius $r_j$. Let $\hat{x}_j$ be the point on this sphere that intersects the ray from $x_j$ towards $x$.  Then the distance $f(x, \hat{x}_j) = |f(x, x_j) - r_j|$. Thus, we can rewrite $g(x)$ as 
\[ 
g(x) = \sum_j \err(f(x, \hat{x}_j)). 
\]
\begin{figure}[t]
  \begin{center}
  \includegraphics[width=3in]{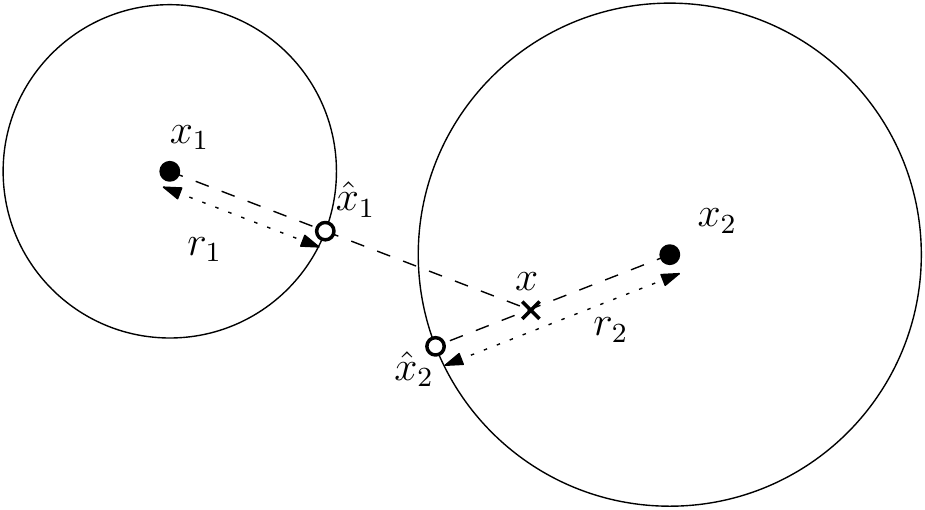}
  \end{center}
\caption{\textsf{A geometric interpretation of the error term $g(x)$.\label{fig:place}}}
\end{figure}

This function is well-known in combinatorial optimization as the min-sum problem. For $\err(\delta) = \delta^2$, $g(x)$ finds the point minimizing the sum-of-squared distances from a collection of fixed points (the $1$-mean), which is the centroid $x^* = \frac{1}{n}\sum_j \hat{x}_j$. For $\err(\delta) = |\delta|$, $g(x)$ finds the $1$-median, the point minimizing the sum of distances from a collection of fixed points.  Although there is no closed form expression for the $1$-median, there are numerous algorithms for solving this problem both exactly~\cite{weiszfeld1937} and approximately~\cite{bose}.  Methods that converge to the global optimum exist for any $\err(\delta) = |\delta|^p, p \le 2$; it is known that if $p$ is sufficiently larger than $2$, then convergent methods may not exist~\cite{brimberg-love}. 

While $g(x)$ can be minimized optimally for error functions $\err$ of interest, the location of the points $\hat{x}_j$ depends on the location of the solution $x^*$, which is itself unknown!  This motivates an alternating optimization procedure, where the current iterate $x$ is used to compute  $\hat{x}_j$, and then these $\hat{x}_j$ are used as input to the min-sum problem to solve for the next value of $x$. 

\begin{algorithm}
\caption{$\place_i(X,D)$}
 \begin{algorithmic}
    \REPEAT 
	\STATE $\epsilon \leftarrow g(x_i)$
      \FOR{$j = 1$ \textbf{to} $n$}
        \STATE $\hat{x}_j \leftarrow $ intersection of sphere of radius $r_j$ around $x_j$ with ray from $x_j$ towards $x_i$
      \ENDFOR
    \STATE $x_i \leftarrow \recenter(\{\hat{x}_1, \hat{x}_2, \ldots, \hat{x}_n\})$
    \UNTIL{$(\epsilon - g(x_i) < t)$} \{for a fixed threshold $t$\}
    \STATE \textbf{return} $x_i$
 \end{algorithmic}
\end{algorithm}

%
%

\subsection{Implementing \recenter}
Up to this point, the description of \PC and \place has been generic, requiring no specification of $\err$ and $f$. In fact, all the domain-specificity of the method appears in \recenter, which solves the min-sum problem. We now demonstrate how different implementations of \recenter allow us to solve the different variants of \mds discussed above.

\subsubsection{The original \mds: \fmds}
Recall from Section~\ref{sec:definitions} that the \fmds problem is defined by $\err(\delta) = \delta^2$ and $f(x, x') = \|x-x'\|_2$. 
Thus, $g(x) = \sum_j \|x - \hat{x}_j\|^2$. As mentioned earlier, the minimum of this function is attained at $x^* = (1/n)\sum_j \hat{x}_j$. Thus, 
$\recenter(\{\hat{x}_1, \hat{x}_2, \ldots, \hat{x}_n \})$ merely outputs $(1/n)\sum_j \hat{x}_j$, and takes $O(n)$ time per invocation. 

\subsubsection{Robust \mds: \rmds}
The robust \mds problem \rmds is defined by $\err(\delta) = |\delta|$ and $f(x, x') = \|x - x'\|_2$. Minimizing the resulting function $g(x)$ yields the famous Fermat-Weber problem, or the $1$-median problem as it is commonly known. An exact iterative algorithm for solving this problem was given by Weiszfeld~\cite{weiszfeld1937}, and works as follows. 
At each step of $\place_i$ the value $x_i$ is updated by 
\[ 
x_i \leftarrow \left. \sum_j \frac{\hat{x}_j}{\|x_i - \hat{x}_j\|} \right/ \sum_j \frac{1}{\|x_i - \hat{x}_j\|}.
\]
This algorithm is guaranteed to converge to the optimal solution~\cite{kuhn73:_probl,ostresh1978}, and in most settings converges quadratically~\cite{Katz74:_local_}. 

\paragraph{Other norms and distances.}
If $\err(\delta) = |\delta|^p, 1 < p < 2$, then an iterative algorithm along the same lines as the Weiszfeld algorithm can be used to minimize $g(x)$ optimally~\cite{brimberg-love}. In practice, this is the most interesting range of values for $p$. It is also known that for $p$ sufficiently larger than $2$, this iterative scheme may not converge. 

We also can tune $\PC$ to the \rsmds problem (using squared distances) by setting $r_j = d_{ij}^2$.

\subsubsection{Spherical \mds}
Spherical \mds poses special challenges for the implementation of \recenter. Firstly, it is no longer obvious what the definition of $\hat{x}_j$ should be, since the ``spheres'' surrounding points must also lie on the sphere. Secondly, consider the case where $\err(\delta) = \delta^2$, and $f(x,x')$ is given by geodesic distance on the sphere. Unlike in the case of $\reals^k$, we no longer can solve for the minimizer of $g(x)$ by computing the centroid of the given points, because this centroid will not in general lie on the sphere, and even computing the centroid followed by a projection onto the sphere will not guarantee optimality. 

The first problem can be solved easily. Rather than draw spheres around each $x_j$, we draw \emph{geodesic spheres}, which are the set of points at a fixed \emph{geodesic} distance from $x_j$. On the sphere, this set of points can be easily described as the intersection of an appropriately chosen halfplane with the sphere.  Next, instead of computing the intersection of this geodesic sphere with the ray from $x_j$ towards the current estimate of $x_i$, we compute the intersection with a \emph{geodesic ray} from $x_j$ towards $x_i$. 

The second problem can be addressed by prior work on computing min-sums on manifolds. Karcher~\cite{karcher1977} proposed an iterative scheme for the geodesic sum-of-squares problem that always converges as long as the points do not span the entire sphere. His work extends (for the same functions $\err, f$) to points defined on more general Riemannian manifolds satisfying certain technical conditions. It runs in $O(n)$ time per iteration.  

For the robust case ($\err(\delta) = |\delta|$), the Karcher scheme no longer works. For this case, we make use of a Weiszfeld-like adaption~\cite{neuro08} that again works on general Riemannian manifolds, and on the sphere in particular. Like the Weiszfeld scheme, this approach takes $O(n)$ time per iteration. 

\subsection{Convergence Proofs}
Here we prove that each step of $\PC$ converges as long as the recursively called procedures reduce the relevant cost functions.  Convergence is defined with respect to a cost function $\kappa$, so that an algorithm converges if at each step $\kappa$ decreases until the algorithm terminates.  

\begin{theorem}
If each call to $\tilde x_i \leftarrow \place_i(X,D)$ decreases the cost $C_i(X,D,x_i)$, then $\PC(D)$ converges with respect to $C(\cdot, D)$.  
\end{theorem}
\begin{proof}
Let $\tilde X \leftarrow \place_i(X, D)$ result from running an iteration of $\place_i(X, D)$. Let $\tilde X = \{x_1, \ldots, x_{i-1}$, $\tilde x_i$, $x_{i+1}, \ldots, x_n\}$.  Then we can argue
\begin{align*}
C(X, D) - C(\tilde X, D)
 &= 
2 \sum_{j=1} \err( f(x_i, x_j) - d_{i,j})  
- 2 \sum_{j=1} \err( f(\tilde{x}_i, x_j) - d_{i,j} ) 
\\  &= 
2 C_i(X, D,x_i) - 2 C_i(\tilde X, D, \tilde x_i) > 0.
\end{align*}
The last line follows because $X$ and $\tilde X$ only differ at $x_i$ versus $\tilde x_i$, and by assumption on $\place_i(X,D)$, this sub-cost function must otherwise decrease.  
\end{proof}

\begin{theorem}
If each call $x_i \leftarrow \recenter(\hat X)$ reduces $\sum_{j=1}^n f(x_i, \hat x_j)^p$, then $\place_i(X,D,x_i)$ converges with respect to $C_i(X,D, \cdot)$.
\end{theorem}
\begin{proof}
First we can rewrite 
\[
C_i(X,D, x_i) 
= 
\sum_{j=1}^n \err( f(x_i, x_j) - d_{i,j} ) 
= 
\sum_{j=1}^n \err ((f(x_i, \hat x_j) + d_{i,j}) - d_{i,j} )
= 
\sum_{j=1}^n \err( f(x_i, \hat x_j) ).
\]
Since $\err(f(x_i, \hat x_j))$ measures the distance to the sphere $\circ_j$, choosing $x_i'$ to minimize (or decrease) $\sum_{j=1}^n \err(f(x_i', \hat x_j))$ must decrease the sum of distances to each point $\hat x_j$ on each sphere $\circ_j$.  Now let $\hat x_j'$ be the closest point to $x_i'$ on $\circ_j$.  $\err(f(x_i', \hat x_j')) \leq \err(f(x_i', x_j))$ and thus 
\[
C_i(X,D, x_i') 
=
\sum_{j=1}^n \err(f(x_i', \hat x_j'))
\leq 
\sum_{j=1}^n \err(f(x_i', \hat x_j'))
\leq 
\sum_{j=1}^n \err(f(x_i, \hat x_j))
 =
C_i(X,D,x_i)
\]
where equality only holds if $x_i = x_i'$, in which case the algorithm terminates.
\end{proof}

\section{Experiments}
\label{sec:expts}
In this section we evaluate the performance of \PC (PC).  Since PC generalizes to many different cost functions, we compare it with the best known algorithm for each cost function, if one exists.  For the \fmds problem the leading algorithm is SMACOF~\cite{dLM09}; for the \rsmds problem the leading algorithm is by Cayton and Dasgupta (CD)~\cite{Cayton}. We know of no previous scalable algorithm designed for \rmds.  We note that the Cayton-Dasgupta algorithm \REE does not exactly solve the \rsmds problem. Instead, it takes a non-Euclidean distance matrix and finds a Euclidean distance matrix that minimizes the error without any rank restrictions.  Thus, as suggested by the authors~\cite{Cayton}, to properly compare the algorithms, we let CD refer to running \REE and then projecting the result to a $k$-dimensional subspace using the SVD technique~\cite{YH38} (our plots show this projection after each step).  With regards to each of these Euclidean measures we compare our algorithm with SMACOF and CD.
We also compare with the popular SVD-based method~\cite{YH38}, which solves the related \cmds problem based on similarities, by seeding all three iterative techniques with the results of the closed-form SVD-based solution.   

Then we consider the family of spherical \mds problems \allsmds.  We compare against a version of SMACOF-Q~\cite{dLM09} that is designed for data restricted to a low dimensional sphere, specifically for the c-2-\smds measure.  We compare this algorithm to ours under the c-2-\smds measure (for a fair comparison with SMACOF-Q) and under the g-1-\smds measure which is the most robust to noise.  

The subsections that follow focus on individual cost measures. We then discuss the overall behavior of our algorithm in Section~\ref{ssec:expt-summary}. 

\paragraph{Data sets, code, and setup.}
\label{ssec:data-gen}
Test inputs for the algorithms are generated as follows. We start with input consisting of a random point set with $n = 300$ points in $\b{R}^d$ for $d=200$, with the target space $T = \b{R}^k$ with $k=10$.  Many data sets in practice have much larger parameters $n$ and $d$, but we limit ourselves to this range  because for larger values CD becomes prohibitively slow. The data is generated to first lie on a $k$-dimensional subspace, and then (full-dimensional) Poisson noise is applied to all points up to a magnitude of 30\% of the variation in any dimension. Finally, we construct the Euclidean distance matrix $D$ which is provided as input to the algorithms. 

These data sets are Euclidean, but ``close'' to $k$-dimensional. To examine the behavior of the algorithms on distance matrices that are non-Euclidean, we generate data as before in a $k$-dimensional subspace and generate the resulting distance matrix $D$. Then we perturb a fraction of the elements of $D$ (rather than perturbing the points) with Poisson noise. The fraction perturbed varies in the set $(2\%, 10\%, 30\%, 90\%)$. 

All algorithms were implemented in MATLAB. For SMACOF, we used the implementation provided by Bronstein~\cite{smacofcode}, and built our own implementation of SMACOF-Q around it. For all other algorithms, we used our own implementation\footnote{All of our code may be found at \url{http://www.cs.utah.edu/~arvind/smds.html}.}. In all cases, we compare performance in terms of the error function $\err$ as a function of clock time. 

\subsection{The \rmds Problem}
Figure \ref{fig:rmds_convergence} shows the cost function $\err$ associated with \rmds plotted with respect to runtime.  
$\PC$ always reaches the best local minimum, partially because only $\PC$ can be adjusted for the \rmds problem.  
We also observe that the runtime is comparable to SMACOF and much faster than CD in order to get to the same \err value.  Although SMACOF initially reaches a smaller cost that PC, it later converges to a larger cost because it optimizes a different cost function (\fmds).

\begin{figure}[tc]
\begin{center}
\subfigure[]{
\includegraphics[width=0.48\textwidth]{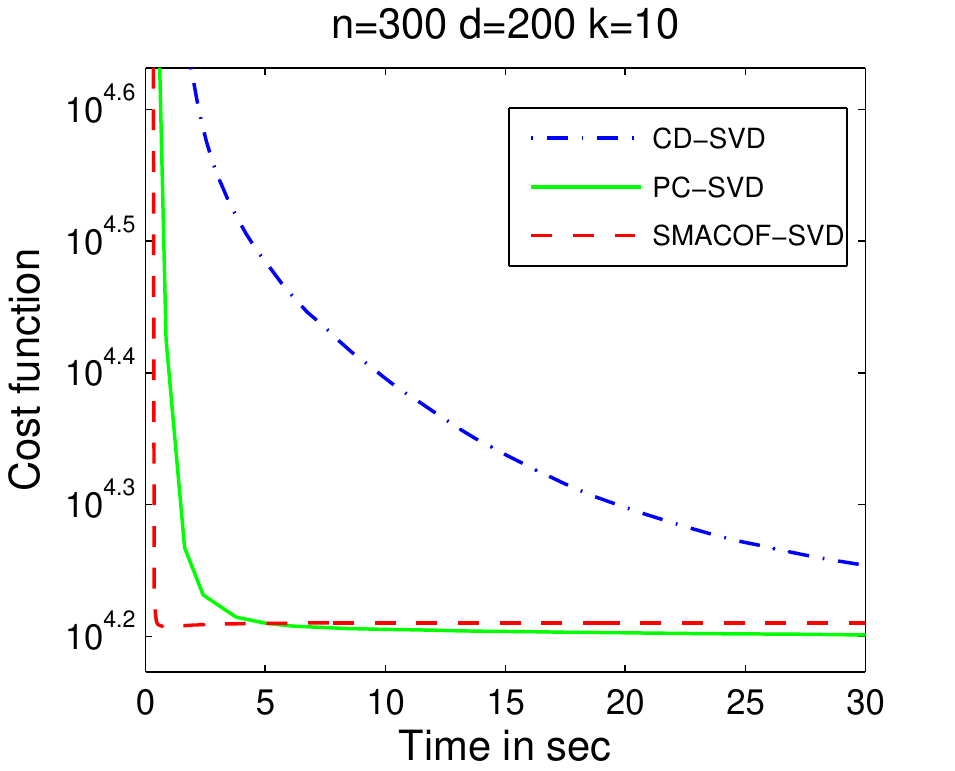}
\label{fig:rmds_convergence}
}
\subfigure[]{
\includegraphics[width=0.48\textwidth]{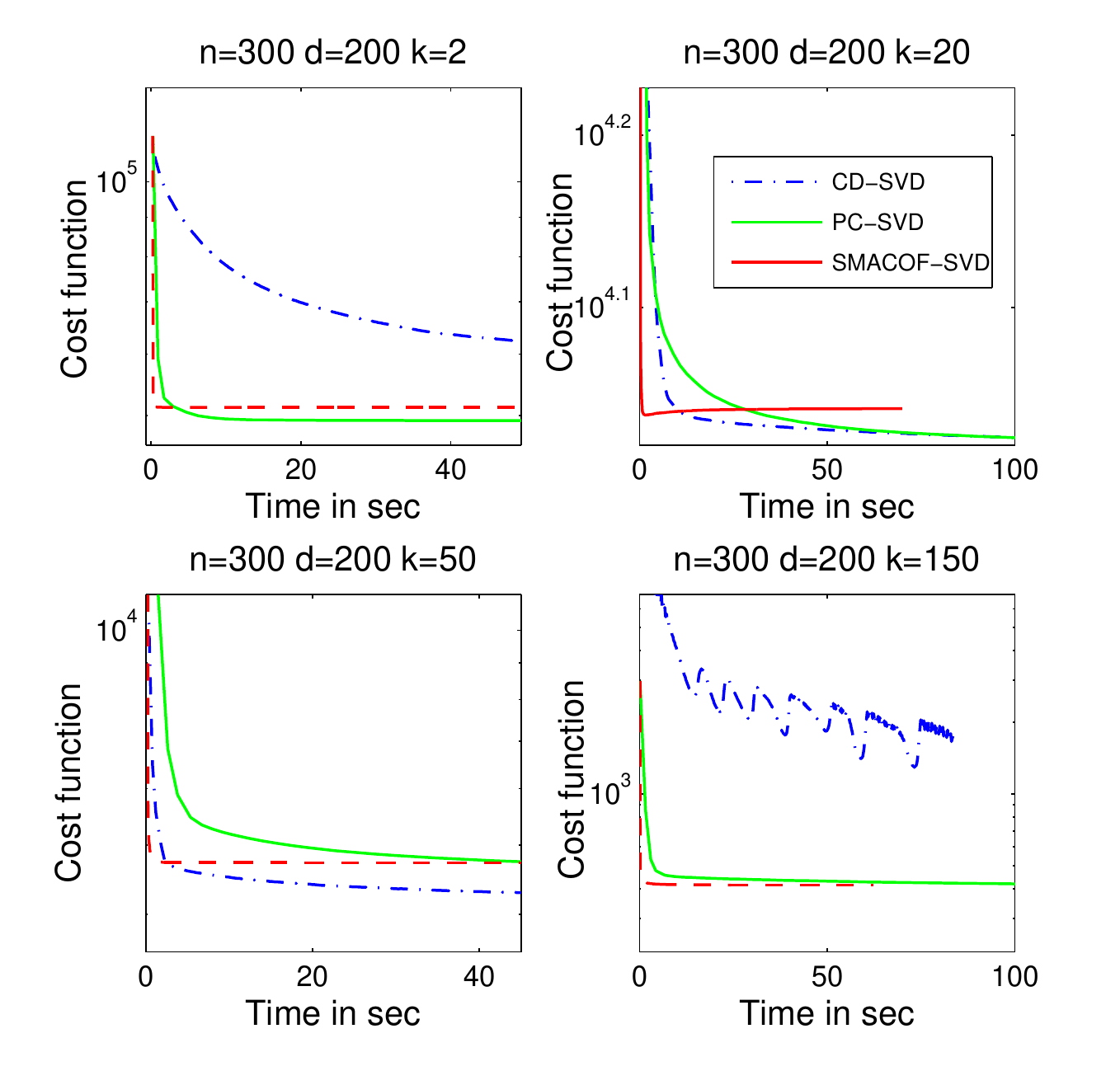}
\label{fig:rmds_changeK}
}
\caption{\textsf{(a) \rmds: Typical behavior of PC, CD and SMACOF.
(b) \rmds: Variation with $k=2,20,50,150$.}}
\end{center}
\end{figure}


We repeat this experiment in Figure \ref{fig:rmds_changeK} for different values of $k$ (equal to $\{2,20,50,150\}$) to analyze the performance as a function of $k$.  Note that PC performs even better for lower $k$ in relation to CD.  This is likely as a result of CD's reliance on the SVD technique to reduce the dimension.  At smaller $k$, the SVD technique has a tougher job to do, and optimizes the wrong metric.  Also for $k=150$ note that CD oscillates in its cost; this is again because the REE part finds a nearby Euclidean distance matrix which may be inherently very high dimensional and the SVD projection is very susceptible to changes in this matrix for such large $k$.
We observe that SMACOF is the fastest method to reach a low cost, but does not converge to the lowest cost value.  The reason it achieves a cost close to that of PC is that for this type of data the \rmds and \fmds cost functions are fairly similar.

In Figure \ref{fig:rmds_noise} we evaluate the effect of changing the amount of noise added to the input distance matrix $D$, as described above.  
We consider two variants of the CD algorithm, one where it is seeded with an SVD-based seed (marked CD+SVD) and one where it is seeded with a  random projection to a $k$-dimensional subspace (marked CD+rand).  In both cases the plots show the results of the REE algorithm after SVD-type projections back to a $k$-dimensional space.  

The CD+SVD technique consistently behaves poorly and does not improve with further iterations.  This probably is because the REE component  finds the closest Euclidean distance matrix which may correspond to points in a much high dimensional space, after which it is difficult for the SVD to help.  The CD+rand approach does much better, likely because the random projection initializes the procedure in a reasonably low dimensional space so REE can find a relatively low dimension Euclidean distance matrix that is nearby.  
SMACOF is again the fastest algorithm, but with more noise, the difference between \fmds and \rmds is larger, and thus SMACOF converges to a configuration with much higher cost than PC.  
We reiterate that PC consistently converges to the lowest cost solution among the different methods, and consistently is either the fastest or is comparable to the fastest algorithm.  
We will see this trend repeated with other cost measures as well.  

\begin{figure}[tc]
\begin{center}
\includegraphics[width=0.48\textwidth]{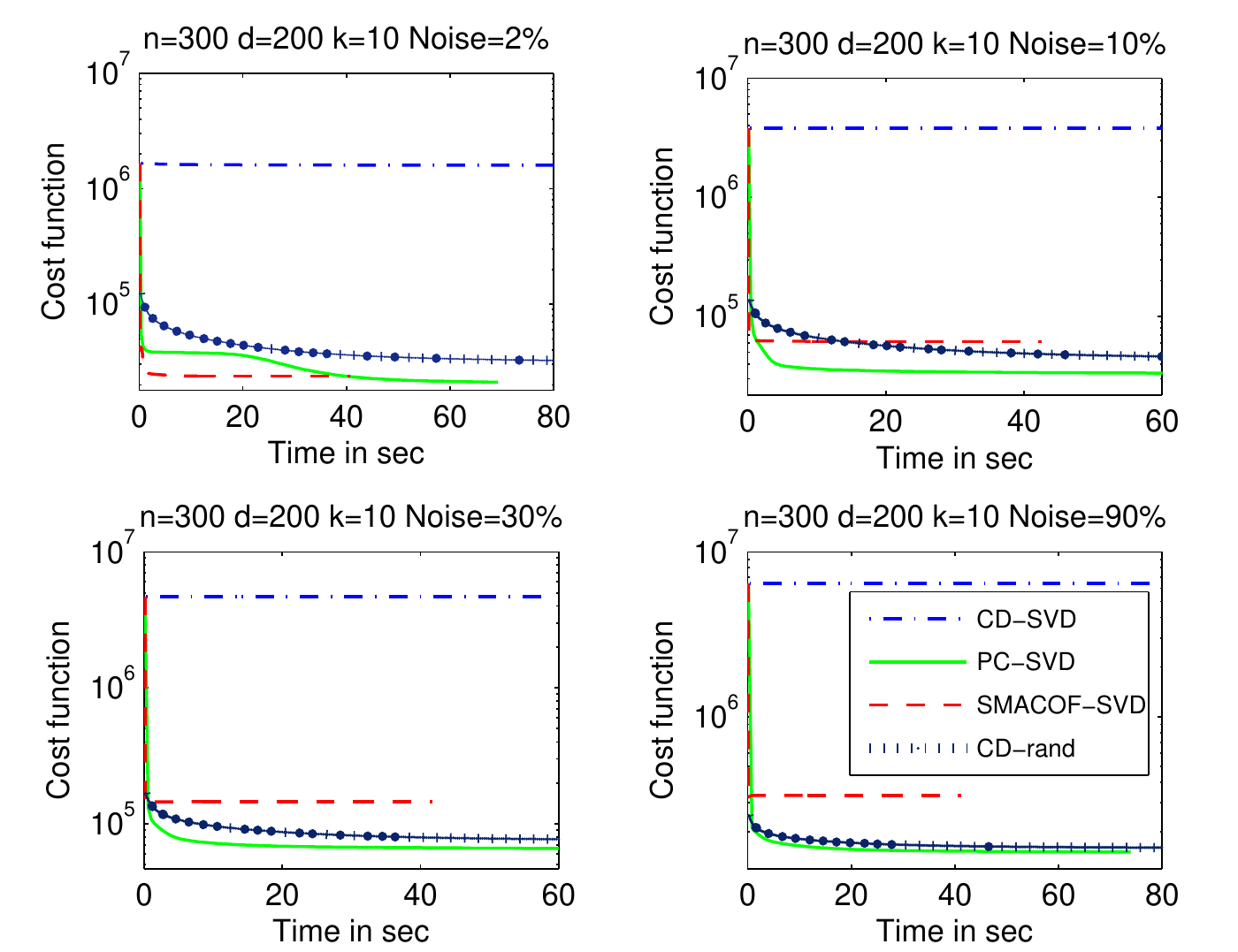}
\caption{\textsf{\rmds: Variation with noise$=2,10,30,90$.\label{fig:rmds_noise}}}
\end{center}
\end{figure}


\subsection{The \fmds Problem}
We next evaluate the algorithms PC, SMACOF, and CD under the \fmds distance measure.  The results are very similar to the \rmds case except now both SMACOF and PC are optimizing the correct distance measure and converge to the same local minimum.  SMACOF is still slightly faster that PC, but since they both run very fast, the difference is of the order of less than a second even in the very worst part of the cost/time tradeoff curve shown in Figure 
\ref{fig:fmds_changeK}.  
Note that CD performs poorly under this cost function here except when $k=50$.  For smaller values of $k$, the SVD step does not optimize the correct distance and for larger $k$ the REE part is likely finding an inherently very high dimensional Euclidean distance matrix, making the SVD projection very noisy.  



\begin{figure}[tc]
\begin{center}
\subfigure[]{
\includegraphics[width=0.48\textwidth]{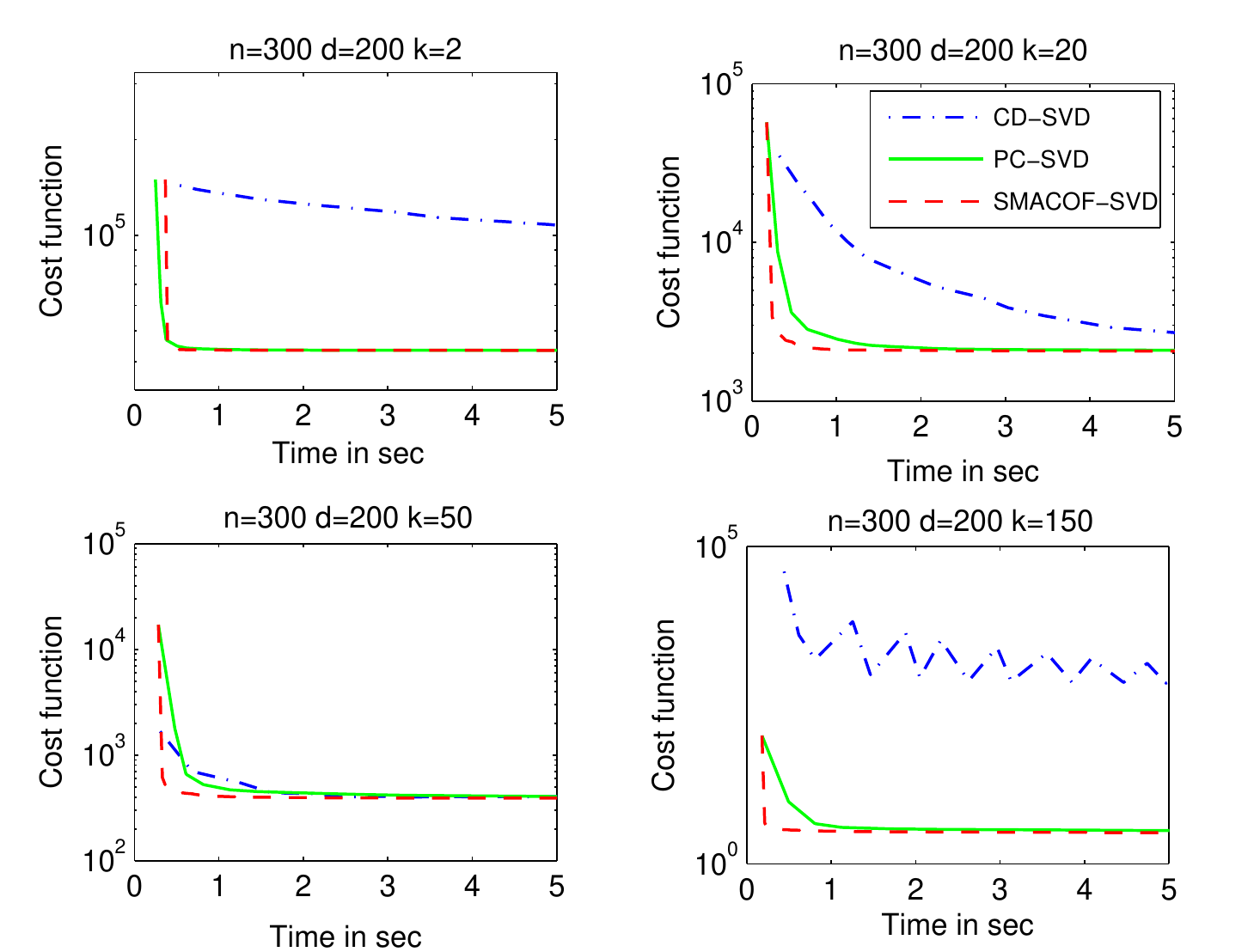}
\label{fig:fmds_changeK}
}
\subfigure[]{
\includegraphics[width=0.48\textwidth]{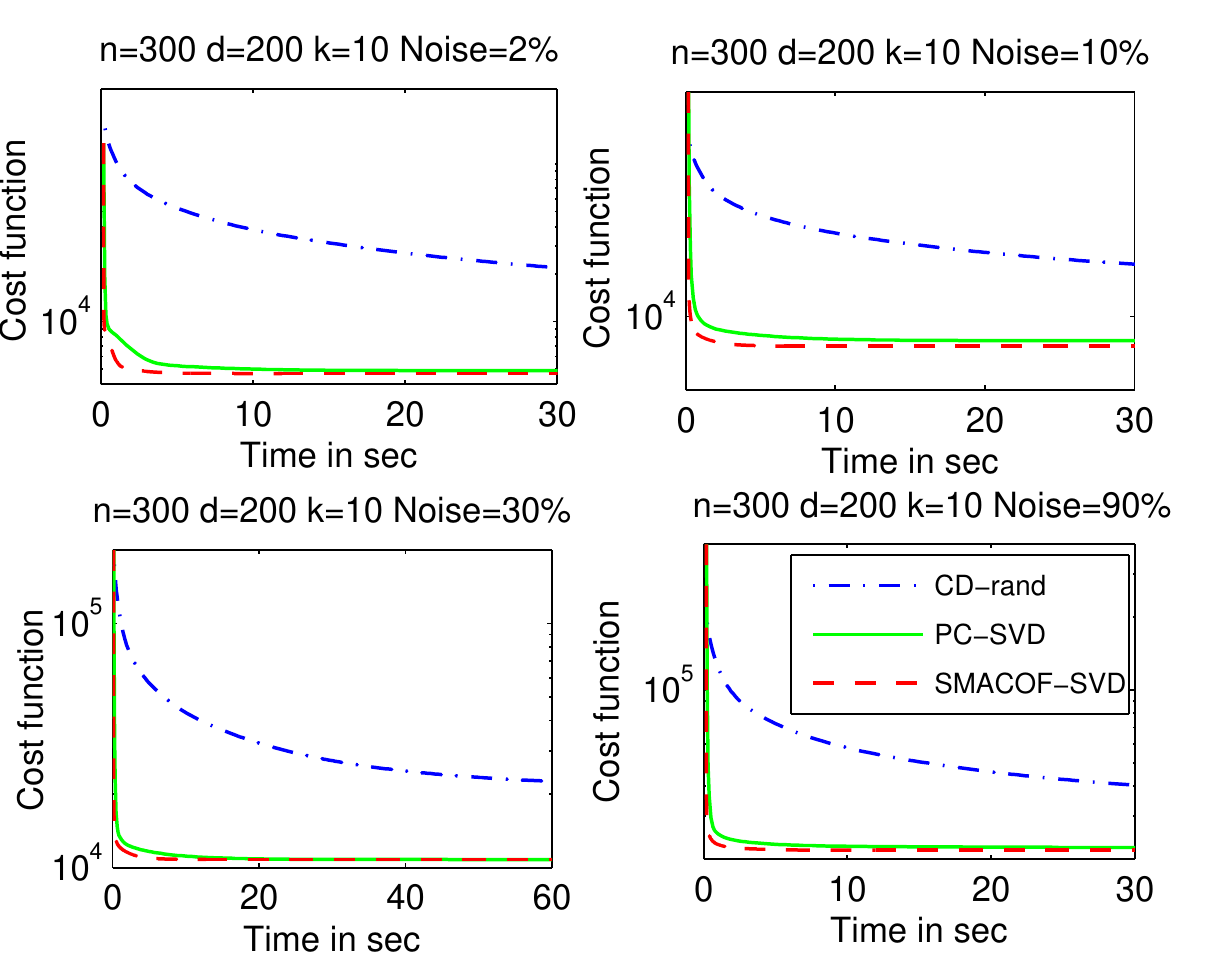}
\label{fig:fmds_noise}
}
\caption{\textsf{(a) \fmds: Variation with $k=2,20,50,150$.
(b) \fmds: Variation with noise$=2,10,30,90$.}}
\end{center}
\end{figure}

For the \fmds measure, SMACOF and PC perform very similarly under different levels of noise, both converging to similar cost functions with SMACOF running a bit faster, as seen in Figure \ref{fig:fmds_noise}.  CD consistently runs slower and converges to a higher cost solution. 


\subsection{The {\large\rsmds} Problem}
In this setting we would expect CD to perform consistently as well as PC because both  minimize the same cost function.  However, this is not always the case because CD requires the SVD step to generate a point set in $\b{R}^k$.  As seen in Figure \ref{fig:r2mds_changeK} this becomes a problem when $k$ is small ($k = 2,10$).  For medium values of $k$, CD converges slightly faster than PC and sometimes to a slightly lower cost solution, but again for large $k$ ($= 150$), the REE part has trouble handling the amount of error and the solution cost oscillates. 
SMACOF is again consistently the fastest to converge, but unless $k$ is very large (i.e. $k=150$) then it converges to a significantly worse solution  because the \fmds and \rsmds error functions are different.  


\begin{figure}[tb!]
\begin{center}
\includegraphics[width=0.48\textwidth]{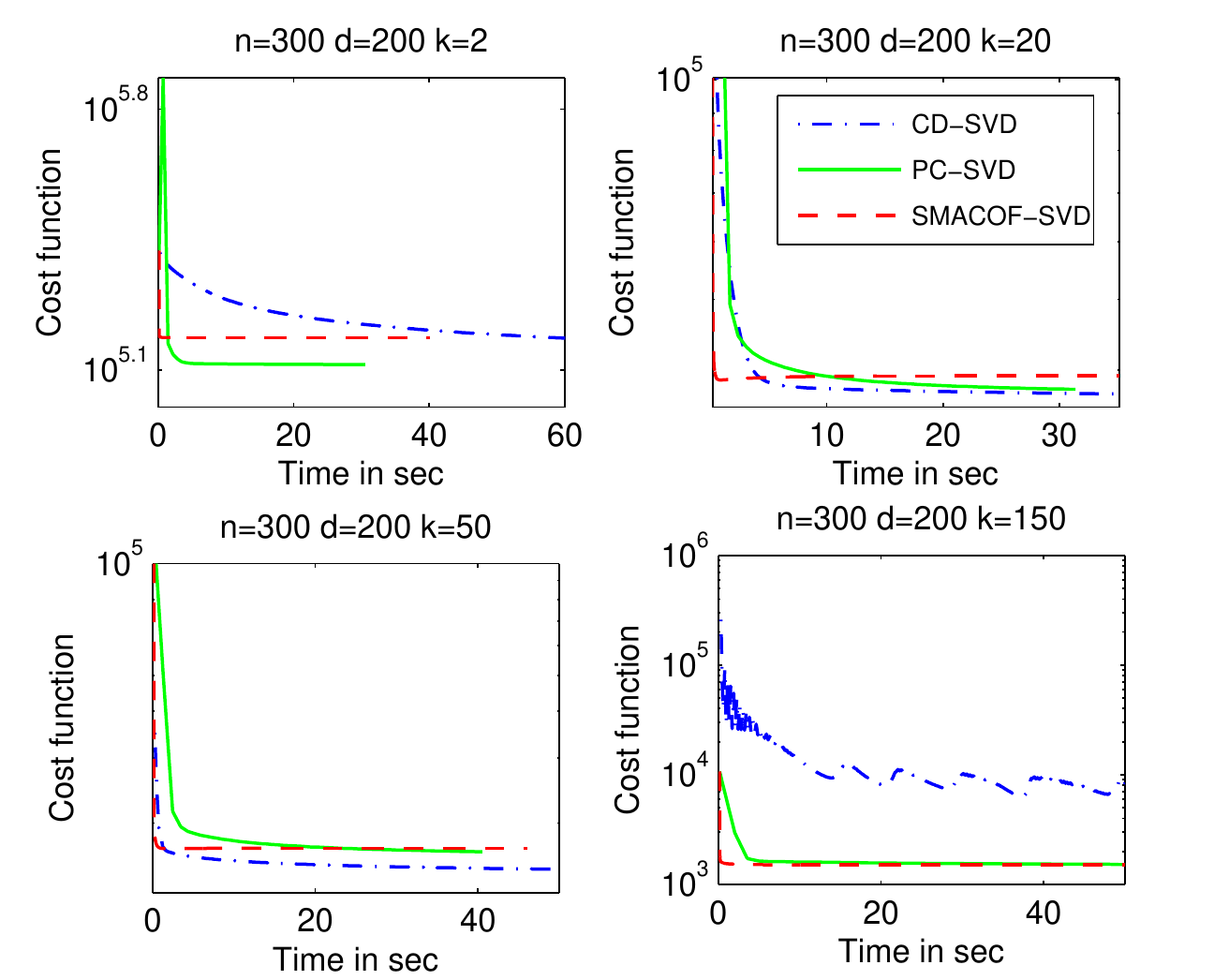}
\caption{\textsf{\rsmds: Variation with $k=2,20,50,150$.}}
\label{fig:r2mds_changeK}
\end{center}
\end{figure}


\subsection{The Spherical \mds Problem}
For the spherical \mds problem we compare PC against SMACOF-Q, an adaptation of SMACOF to restrict data points to a low-dimensional sphere, and a technique of Elad, Keller and Kimmel~\cite{DBLP:conf/scalespace/ElbazKK05}.  It turns out that the Elad \emph{et.al.} approach consistently performs poorly compared to both other techniques, and so we do not display it in our reported results. SMACOF-Q basically runs SMACOF on the original data set, but also adds one additional point $p_0$ at the center of the sphere.  The distance $d_{0,i}$ between any other point $p_i$ and $p_0$ is set to be $1$ thus encouraging all other points to be on a sphere, and this constraint is controlled by a weight factor $\kappa$, a larger $\kappa$ implying a stronger emphasis on satisfying this constraint. Since the solution produced via this procedure may not lie on the sphere, we normalize all points to the sphere after each step for a fair comparison.  

Here we compare PC against SMACOF-Q in the g-1-\smds (Figure \ref{fig:l1Geodesic}) and the c-2-\smds (Figure \ref{fig:l2Chordal}) problem.  
For g-1-\smds, PC does not converge as quickly as SMACOF-Q with small $\kappa$, but it reaches a better cost value.  However, when SMACOF-Q is run with a larger $\kappa$, then PC runs faster and reaches nearly the same cost value.  For our input data, the solution has similar g-1-\mds and c-1-\mds cost.  When we compare SMACOF-Q with PC under c-2-\mds (Figure \ref{fig:l2Chordal}) then for an optimal choice of $\kappa$ in SMACOF-Q, both PC and SMACOF-Q perform very similarly, converging to the same cost function and in about the same time.  But for larger choices of $\kappa$ SMACOF-Q does much worse than PC.

\begin{figure}[htbp]
\begin{center}
\subfigure[]{
\includegraphics[width=0.48\textwidth]{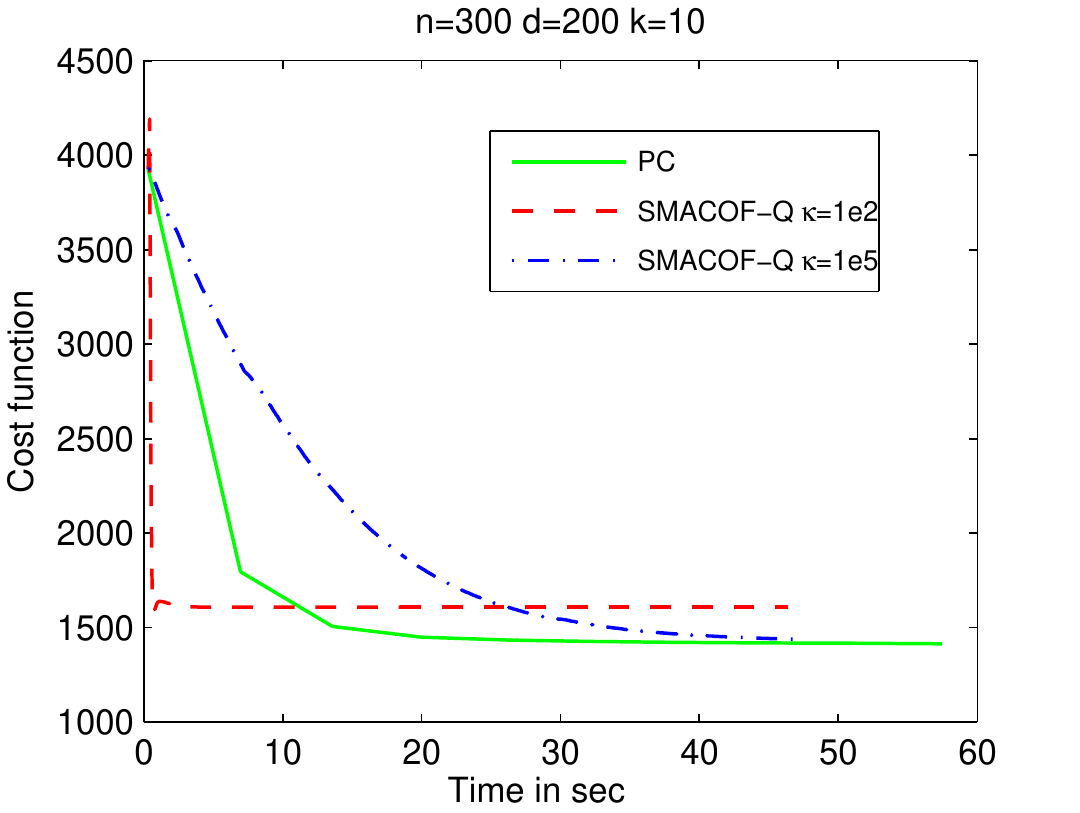}
\label{fig:l1Geodesic}
}
\subfigure[]{
\includegraphics[width=0.48\textwidth]{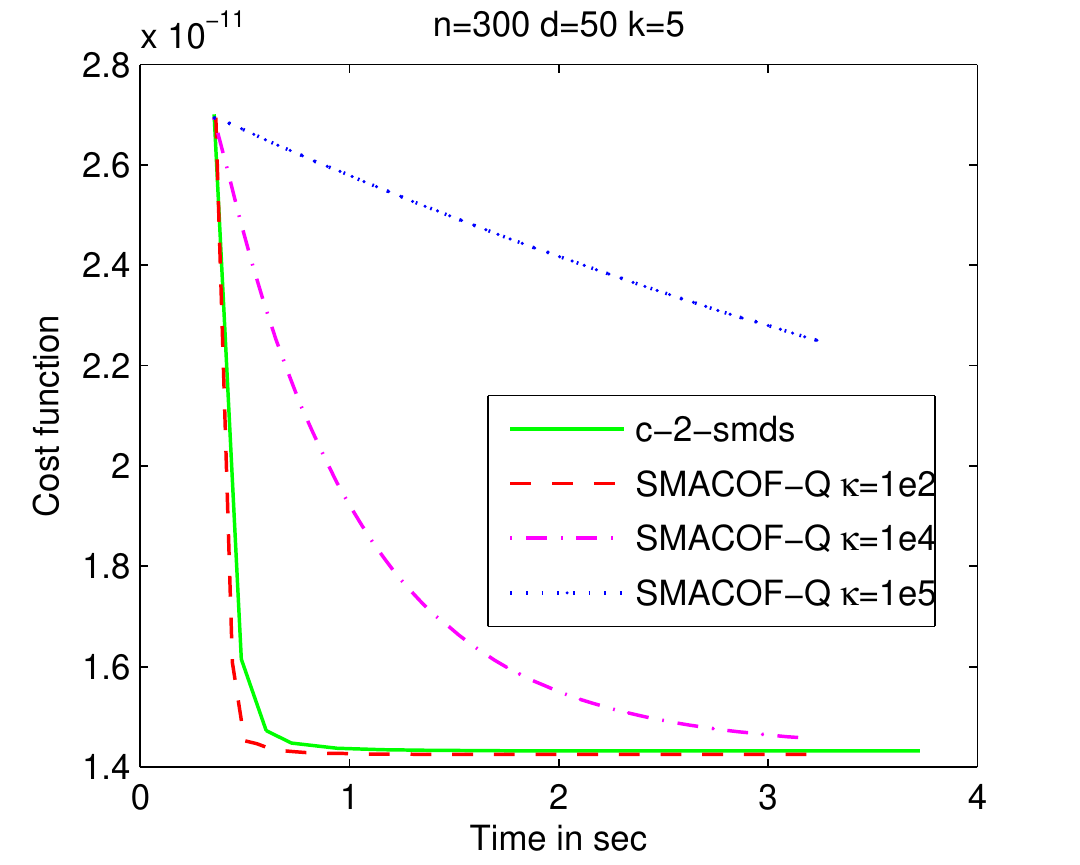}
\label{fig:l2Chordal}
}
\caption{\textsf{(a) g-1-\smds: Comparing PC with SMACOF-Q for different values of penalty parameter $\kappa$.  
(b) c-2-\smds: Comparing PC with SMACOF-Q for different values of penalty parameter $\kappa$.
}}
\end{center}
\end{figure}



In both cases, it is possible to find a value of $\kappa$ that allows SMACOF-Q to match PC. However, this value is different for different settings, and varies from input to input. The key observation here is that since PC is \emph{parameter-free}, it can be run regardless of the choice of input or cost function, and consistently performs well.

\subsection{Summary Of Results}
\label{ssec:expt-summary}

In summary, here are the main conclusions that can be drawn from this experimental study. Firstly, PC is consistently among the top performing methods, regardless of the choice of cost function, the nature of the input, or the level of noise in the problem. Occasionally, other methods will converge faster, but will not in general return a better quality answer, and different methods have much more variable behavior with changing inputs and noise levels.

\section{A JL Lemma for Spherical Data}
\label{sec:sphere}
In this section we present a Johnson-Lindenstrauss-style bound for mapping data from a high dimensional sphere to a low-dimensional sphere while preserving the distances to within a multiplicative error of $(1+\epsilon)$.

Consider a set $Y \subset \b{S}^{d} \subset \b{R}^{d+1}$ of $n$ points, defining a distance matrix $D$ where the element $d_{i,j}$ represents the geodesic distance between $y_i$ and $y_j$ on $\b{S}^k$.  
We seek an embedding of $Y$ into $\b{S}^d$ that preserves pairwise distances as much as possible.  For a set $Y \in \b{S}^d$ and a projection $\pi(Y) = X \subset \b{S}^k$ we say the $X$ has \emph{$\gamma$-distortion} from $Y$ if these exists a constant $c$ such that for all $x_i, x_j \in X$
$$ (1-\gamma)f(y_i, y_j) \leq c f(x_i, x_j) \leq (1+\gamma)f(y_i, y_j). $$

For a subspace $H = \b{R}^k$, let $\pi_H(Y)$ be the projection of $Y \in \b{R}^d$ onto $H$ and then scaled by $d/k$.  For $X \in \b{R}^k$, let $S(X)$ be the projection to $\b{S}^{k-1}$, that is for all $x \in X$, the corresponding point in $S(X)$ is $x / ||x||$.  

When $f(y_i, y_j) = ||y_i - y_j||$, and $Y \in \b{R}^d$, then the Johnson-Lindenstrauss (\JL) Lemma~\cite{JL84} says that if $H \subset \b{R}^d$ is a random $k$-dimensional linear subspace with $k = O((1/\eps^2) \log (n/\delta))$, then $X = \pi_H(Y)$ has $\eps$-distortion from $Y$ with probability at least $1-\delta$.  

We now present the main result of this section. 
We note that recent results~\cite{AHY07} have shown similar results for point on a variety of manifolds (including spheres) where projections preserve Euclidean distances.  We reiterate that our results extend this to geodesic distances on spheres which can be seen as angle $\angle_{x,y}$ between the vectors to points $x, y \in \b{S}^k$.  Another recent result~\cite{Mag02} shows that $k = O((1/\eps^2) \log(n/\delta))$ dimensions preserves $\sqrt{\eps}$-distortion in angles, which is weaker than the following result. 

\begin{theorem}
Let $Y \subset \b{S}^d \subset \b{R}^{d+1}$, and let $H = \b{R}^{k+1}$ be a random subspace of $\b{R}^d$ with $k = O((1/\eps^2) \log (n/\delta))$ with $\eps \in (0,1/4]$.  Let $f(y_i, y_j)$ measure the geodesic distance on $\b{S}^d$ (or $\b{S}^k$ as appropriate).  
Then $S(\pi_H(Y))$ has $\eps$-distortion from $Y$ with probability at least $1-\delta$.
\label{thm:JL-sphere}
\end{theorem}

This implies that if we project $n$ data points that lie on any high-dimensional sphere to a low-dimensional sphere $\b{S}^k$ with $k \sim \log n$, then the pairwise distances are each individually preserved.  Before we proceed with the proof, we require a key technical lemma.  

\begin{lemma}
For $\eps \in [0,0.5]$ and $x \in [0,0.7]$, 
\begin{enumerate}
  \item[(1)] $\sin((1-2\eps)x) \leq (1-\eps)\sin (x)$, and
  \item[(2)] $\sin((1+2\eps)x) \geq (1+\eps)sin(x)$.
\end{enumerate}
\label{lem:small-angle}
\end{lemma}
\begin{proof}
  Let $g_\eps(x) = (1-\eps)\sin x - \sin((1-2\eps)x)$. We will show that for $x \in [0,1]$ and $\eps \in [0,0.5]$, $g_\eps(x)$ is concave. This implies that it achieves its minimum value at the boundary. Now $g_\eps(0) = 0$ for all $\eps$, and it can be easily shown that $g_\eps(0.7) \ge 0$ for $\eps \in [0, 0.5]$. This will therefore imply that $g_\eps(x) \ge 0$ in the specified range. 

It remains to show that $g_\eps(x)$ is concave in $[0,0.7]$. 
\begin{align*}
  g''_\eps(x) &= (1-2\eps)^2 \sin((1-2\eps)x) -(1-\eps)\sin x \\ 
&\le (1-\eps)(\sin((1-2\eps)x) - \sin x)
\end{align*}
which is always negative for $\eps \in [0,0.5]$ and since $\sin x$ is increasing in the range $[0,0.7]$. 

This proves the first part of the lemma. For the second part, observe that $h_\eps(x) = \sin((1+2\eps)x) - (1+\eps)sin(x)$ can be rewritten as $h_\eps(x) = g_{-\eps}(-x)$. The rest of the argument follows along the same lines, by showing that $h_\eps(x)$ is concave in the desired range using that $h''_\eps(x) = g''_{-\eps}(-x)$. 

While the upper bound of $0.7$ on $x$ is not tight, it is close. The actual bound (evaluated by direct calculation) is slightly over $0.72$.
\end{proof}

\begin{proof}[Proof of Theorem \ref{thm:JL-sphere}]
Let $X = \pi_H(Y)$.  We consider two cases, (\emph{Short Case}) when $\|y_i - y_j\| \leq 1/2$ and (\emph{Long Case}) when $\|y_i - y_j\| \in (1/2, 2]$.  

\emph{Short Case: }
First consider points $y_i, y_j \in \b{S}^d$ such that $||y_i - y_j|| \leq 1/2$.  
Note that $||y_i - y_j|| = 2 \sin(\angle_{y_i, y_j}/2)$, since $||y_i|| = ||y_j|| = 1$.  
By \JL, we know that there exists a constant $c$ such that 
$$
(1-\eps/8)||y_i - y_j|| \leq c ||x_i - x_j|| \leq (1+\eps/8)||y_i - y_j||.
$$  
 
\begin{figure}
\begin{center}
\includegraphics{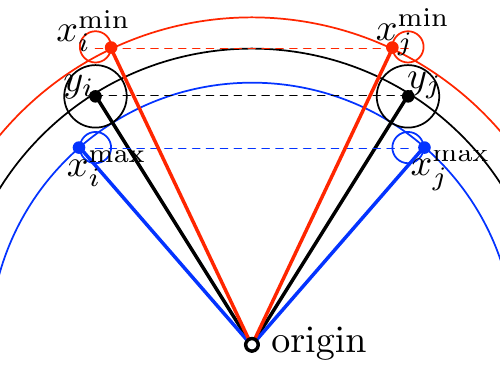}
\end{center}
\caption{\textsf{\label{fig:min-max}
Illustration of the bounds on $\angle_{x_i,x_j}$ when $||y_i - y_j|| \leq 1/2$.  The angle $\angle_{x_i^{\max}, x_j^{\max}}$ is the largest when $||x_i^{\max}|| = ||x_i^{\max}||$ is as small as possible (lies on inner circle) and $||x_i^{\max} - x_j^{\max}||$ is as large as possible (on the outer edges of the disks of diameter $\eps/8$ shifted down from dashed line of length $||y_i - y_j||$.  Bounds for $x_i^{\min}$ and $x_j^{\min}$ are derived symmetrically.}}
\end{figure}
 
We need to compare the angle $\angle_{x_i, x_j}$ with that of
$\angle_{y_i, y_j}$.  The largest $\angle_{x_i, x_j}$ can be is when
$c||x_i|| = c||x_j|| = (1-\eps/8)$ is as small as possible, and so $||cx_i - cx_j|| = (1+\eps/8)||y_i - y_j||$ is as large as
possible.  See Figure \ref{fig:min-max}.  In this case, we have 
\begin{eqnarray*}
(||c x_i||  + ||c x_j||)\sin(\angle_{x_i, x_j}/2) &\leq& ||c x_i - c x_j|| 
\\
2(1-\eps/8) \sin (\angle_{x_i, x_j}/2) &\leq& (1+\eps/8) ||y_i - y_j||
\\
2(1-\eps/8) \sin(\angle_{x_i, x_j}/2) &\leq& (1+\eps/8) 2\sin(\angle_{y_i, y_j}/2) 
\\
\sin(\angle_{x_i, x_j}/2)  &\leq& \frac{1+\eps/8}{1-\eps/8} \sin(\angle_{y_i, y_j}/2),
\end{eqnarray*}
which for $\eps < 4$ implies
$$
\sin(\angle_{x_i, x_j}/2) \leq (1+\eps/2) \sin(\angle_{y_i,y_j}/2).
$$
Similarly, we can show when $\angle_{x_i, x_j}$ is as small as possible (when $c||x_i|| = c||x_j|| = (1+\eps)$ and $||cx_i - c x_j|| = (1-\eps)||y_i - y_j||$), then 
$$
(1-\eps/2) \sin(\angle_{y_i, y_j}/2) \leq \sin(\angle_{x_i, x_j}/2).
$$

We can also show (via Lemma \ref{lem:small-angle}) that since $||y_i - y_j|| \leq 1/2$ implies $\angle_{y_i, y_j} < 0.7$ we have 
\[ \sin((1-\eps) \angle_{y_i, y_j}) \leq (1-\eps/2) \sin (\angle_{y_i,  y_j}) \]
and   
\[
(1+\eps/2)\sin(\angle_{y_i, y_j}) \leq \sin ((1+\eps)\angle_{y_i, y_j}).
\]  

Thus, we have 
\[
\begin{array}{rcccl}
\sin((1-\eps)\angle_{y_i,y_j}/2) &\leq& \sin(\angle_{x_i, x_j}/2) &\leq& \sin((1+\eps)\angle_{y_i, y_j}/2)
\\
(1-\eps)\angle_{y_i,y_j}/2 &\leq& \angle_{x_i, x_j}/2 &\leq& (1+\eps)\angle_{y_i, y_j}/2
\\
(1-\eps)\angle_{y_i,y_j} &\leq& \angle_{x_i, x_j} &\leq& (1+\eps)\angle_{y_i, y_j}.
\end{array}
\]

\emph{Long Case: }
For $||y_i - y_j|| \in (1/2, 2]$, we consider $6$ additional points $y^{(h)}_{i,j} \in \b{S}^{d+1}$ (for $h \in [1:6]$) equally spaced between $y_i$ and $y_j$ on the shortest great circle connecting them.  Let $\hat Y$ be the set $Y$ plus all added points $\{ y^{(h)}_{i,j}\}_{h = [1:6]}$.  Note that $|\hat Y| = O(n^2)$, so by \JL we have that 
\[
(1-\eps/8)||y_i - \hat{y}_{i,j}|| \leq c ||x_i - \hat x_{i,j}|| \leq (1+\eps/8) ||y_i - \hat y_{i,j}||.
\]
For notational convenience let $y_i = y^{(0)}_{i,j}$ and $y_j = y^{(7)}_{i,j}$.  
Since for $\|y_i - y_j\| \in (1/2, 2]$ then $\|y^{(h)}_{i,j} - y^{(h+1)}_{i,j}\| \leq 1/2$, for $h \in [0:6]$.  This follows since the geodesic length of the great circular arc through $y_i$ and $y_j$ is at most $\pi$, and $\pi/7 < 1/2$.  Then the chordal distance for each pair $\|y^{(h)}_{i,j} - y^{(h+1)}_{i,j}\|$ is upper bounded by the geodesic distance.     
Furthermore, by invoking the short case, for any pair 
\[
(1-\eps)\angle_{y^{(h)}_{i,j},y^{(h+1)}_{i,j}}  \leq  \angle_{x^{(h)}_{i,j}, x^{(h+1)}_{i,j}}  \leq  (1+\eps)\angle_{y^{(h)}_{i,j}, y^{(h)}_{i,j}}.
\]
Then since projections preserve coplanarity (specifically, the points $0$ and $y^{(h)}_{i,j}$ for $h \in [0:7]$ are coplanar, hence $0$ and $x^{(h)}_{i,j}$ for $h \in [0:7]$ are coplanar), we can add together the bounds on angles which all lie on a single great circle.
\begin{eqnarray*}
(1-\eps) \angle_{y_i, y_j} 
\leq 
(1-\eps)\sum_{h=0}^6 \angle_{y^{(h)}_{i,j}, y^{(h+1)}_{i,j}}  
\leq  &
\sum_{h=0}^6 \angle_{x^{(h)}_{i,j}, x^{(h+1)}_{i,j}}
&\leq 
(1+\eps) \sum_{h=0}^6 \angle_{y^{(h)}_{i,j}, y^{(h+1)}_{i,j}}
\leq
\min\{\pi,  (1+\eps)\angle_{y_i, y_j}\} 
\\
(1-\eps) \angle_{y_i, y_j} 
\leq  &
\angle_{x_i, x_j}  
&\leq 
\min\{\pi,  (1+\eps)\angle_{y_i, y_j}\}. \hfill \qedhere
\end{eqnarray*}
\end{proof}

\bibliography{ref}

\bibliographystyle{abbrv}
\end{document}